\documentclass[journal]{IEEEtran}
\usepackage{amsmath,amssymb,amsfonts,amsthm,mathtools}
\usepackage{algorithmic}
\usepackage{algorithm}
\usepackage{array}
\usepackage[caption=false,font=normalsize,labelfont=sf,textfont=sf]{subfig}
\usepackage{textcomp}
\usepackage{stfloats}
\usepackage{xurl}
\usepackage{verbatim}
\usepackage{graphicx}
\usepackage{needspace}
\usepackage[numbers,sort&compress]{natbib}
\usepackage{booktabs}
\usepackage{multirow}
\usepackage{siunitx}
\usepackage{bm}
\newtheoremstyle{namedplain}
    {\topsep}{\topsep}
    {\itshape}{}{\bfseries}{.}{5pt plus 1pt minus 1pt}
    {\thmname{#1}\thmnumber{ #2}\thmnote{: #3}}
\theoremstyle{namedplain}
\newtheorem{theorem}{Theorem}
\newtheorem{lemma}{Lemma}

\newtheorem{proposition}{Proposition}
\theoremstyle{remark}
\newtheorem{remark}{Remark}

\newcommand{\KL}{\mathrm{KL}}
\newcommand{\kl}{\mathrm{kl}}
\newcommand{\E}{\mathbb{E}}
\newcommand{\R}{\mathbb{R}}
\newcommand{\N}{\mathcal{N}}

\renewcommand{\eqref}[1]{Eq.~\ref{#1}}

\renewcommand{\thesection}{\arabic{section}}
\renewcommand{\thesectiondis}{\Roman{section}.}

\renewcommand{\paragraph}[1]{\par\vspace{3pt}\noindent#1\hspace{0.6em}\ignorespaces}

\begin{document}

\title{MAGDrop: Momentum Adaptive Gradient Dropout\\
With Provable Generalization Guarantees}

\author{Adeel~Safder and Bum~Jun~Kim
\thanks{Adeel Safder is with Quaid-i-Azam University, Islamabad, Pakistan (ORCID: 0009-0000-7496-9058).}
\thanks{Bum Jun Kim is with the Graduate School of Engineering, The University of Tokyo, Tokyo, Japan (e-mail: bumjun.kim@\allowbreak weblab.t.u-tokyo.ac.jp).}
\thanks{Corresponding author: Bum Jun Kim.}}

\markboth{Preprint}%
{Safder and Kim: MAGDrop: Momentum Adaptive Gradient Dropout With Provable Generalization Guarantees}

\maketitle

\begin{abstract}
Standard dropout applies a fixed, uniform noise rate to all units,
ignoring available gradient information.
We propose momentum adaptive gradient dropout (MAGDrop), a Bayesian
linear layer that adapts per-unit weight noise by
tracking an exponential moving average of gradient magnitudes.
Units with large recent gradient magnitudes receive less
noise to preserve learned features; units with below-average
gradient magnitudes receive more noise to promote exploration.
This adaptation transforms the classifier head into a
data-dependent Bayesian predictor whose
mean predictor has no additional inference cost. Our primary contribution is a
certificate-aware low-rank posterior, MAGDrop-Subspace, that
addresses the tension between strong accuracy and tight bounds
from probably approximately correct (PAC)-Bayes theory.
A full-head stochastic posterior incurs a Kullback--Leibler (KL)
divergence of approximately 8,625 nats. By confining the posterior
to an $r$-dimensional subspace
constructed from an independent prior-construction sample $S_0$
and certifying on an independent sample $S$ with $S_0\perp S$,
MAGDrop-Subspace reduces KL by approximately a factor of 800
while keeping accuracy close
to the baseline. On CIFAR-10 with a 13-layer convolutional neural
network and $r=50$, MAGDrop-Subspace achieves a PAC-Bayes Gibbs-risk
bound of 0.1882 at 90\% confidence using binary KL inversion,
compared with 0.4439 for the full-head baseline.
Mean-predictor accuracy is 82.31\%, compared with 83.74\%
for the baseline. Complementary accuracy experiments on
ImageNet-1K show that MAGDrop reaches top-1 accuracy similar
to the strong baseline: 78.57\% compared with 78.51\%
on the base Vision Transformer with $16\times16$ patches,
and 82.14\% compared with 82.23\% on
ConvNeXt-Base.

\end{abstract}


\section{Introduction}
\label{sec:intro}
\IEEEPARstart{G}{eneralization} guarantees for deep neural networks remain
one of the central open problems in machine learning.
Standard training produces models with strong empirical
performance, but without any formal certificate that this
performance will hold on unseen data drawn from the same
distribution. Probably approximately correct (PAC)-Bayes theory~\citep{mcallester1999pac}
provides such certificates for randomized predictors via
a bound on the Kullback--Leibler (KL) divergence between a learned posterior
and a fixed prior. However, applying this framework to
modern deep networks typically yields loose bounds:
Bayesian methods that inject sufficient noise to make the
bound meaningful tend to degrade accuracy, while methods
that preserve accuracy produce posteriors too close to
degenerate point masses for the bound to be tight.

\begin{figure*}[t!]
\centering
\includegraphics[width=\textwidth]{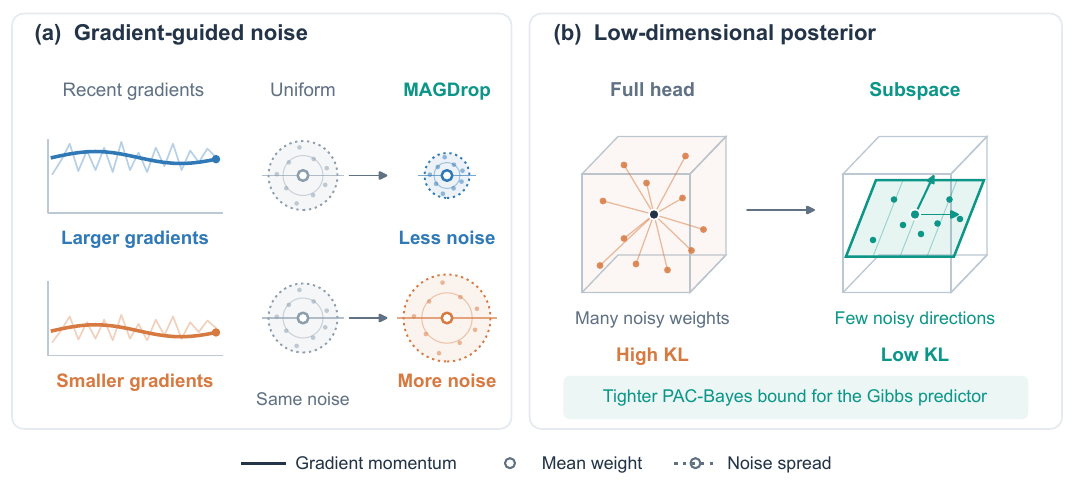}
\caption{Overview of MAGDrop and MAGDrop-Subspace.
(a) Gradient-adaptive noise; circle widths show spread around
mean weights. (b) Low-dimensional posterior for Gibbs-risk
certification.}
\label{fig:magdrop_overview}
\end{figure*}

We propose momentum adaptive gradient dropout (MAGDrop),
which uses gradient magnitude as a readily available proxy
for recent per-unit learning activity. An exponential moving
average (EMA) of gradient magnitudes reduces Gaussian weight
noise for active units and increases it for less active units.
The variance multipliers follow these statistics directly,
without a variational objective for the variance.
Figure~\ref{fig:magdrop_overview} illustrates this adaptive
noise mechanism and the low-dimensional posterior used by
MAGDrop-Subspace to tighten the generalization certificate.

Adaptive noise alone does not resolve the large KL cost of a
full-head posterior. MAGDrop-Subspace confines posterior
variation to a low-dimensional affine subspace, with the prior,
basis, and deterministic components constructed independently
of the certification sample. This separates noise allocation
from the control of posterior complexity and permits non-vacuous
PAC-Bayes-$\kl$ certificates. Compression-based methods can also
certify deterministic networks by training and quantizing
weights in a low-dimensional subspace~\citep{lotfi2022compression}.
MAGDrop-Subspace instead certifies a continuous Gaussian
posterior through its analytic KL divergence from a prior on
the same support. The certificate bounds Gibbs risk; reported
accuracy uses the separate mean predictor.

\Needspace{4\baselineskip}
\paragraph{Contributions.}
\begin{enumerate}
    \item We introduce MAGDrop, a Bayesian linear layer with
          bounded Gaussian noise adapted from gradient momentum
          and a mean predictor with no additional inference cost.

    \item We develop MAGDrop-Subspace, a low-dimensional posterior
          with a shared-support prior and analytic KL, together
          with PAC-Bayes generalization certificates under an
          independent prior-construction and certification split.

    \item We evaluate certification on CIFAR-10 and accuracy on
          ImageNet-1K. Controlled comparisons of random QR and
          principal component analysis (PCA) bases relate posterior
          sensitivity to curvature, while variance ablations
          distinguish the effects of subspace structure and
          gradient-based noise allocation.
\end{enumerate}

\section{Related Work}
\label{sec:related}

\paragraph{Dropout and Bayesian approximations.}
Standard dropout~\citep{srivastava2014dropout} masks activations
at a fixed rate, while Monte Carlo (MC) dropout~\citep{gal2016dropout}
interprets this mechanism as approximate Bayesian inference.
DropConnect~\citep{wan2013regularization} instead applies
Bernoulli masks to weights. Variational
Dropout~\citep{kingma2015variational} and Sparse Variational
Dropout~\citep{molchanov2017variational} learn Gaussian noise
levels; Concrete Dropout~\citep{gal2017concrete} learns masking
probabilities through a continuous relaxation.
MAGDrop updates bounded Gaussian variance multipliers from
gradient statistics without optimizing a variational objective
for those multipliers. Our certificates use proper Gaussian
priors sharing the posterior's full-head or latent support,
satisfying the prior and support conditions discussed by
\citet{hron2017variational}.

\paragraph{Adaptive regularization and gradient-guided dropout.}
Dropout admits an adaptive-regularization interpretation~\citep{wager2013dropout};
adaptive dropout can also learn probabilities through an
auxiliary network~\citep{ba2013adaptive}.
Dynamic DropConnect~\citep{yang2025dynamic} uses normalized
absolute gradients to set per-edge Bernoulli drop probabilities,
preferentially retaining large-gradient connections without
extra learned parameters. MAGDrop shares this heuristic but
uses an EMA and bounded Gaussian variances per unit or latent
direction. For domain generalization,
\citet{michalkiewicz2023domain} discard components with high
gradient signal-to-noise ratios and meta-learn dropout ratios;
MAGDrop instead assigns less noise to larger gradient
magnitudes and certifies in-distribution Gibbs risk.
The variational version of adaptive moment estimation (Adam),
called variational Adam (Vadam)~\citep{khan2018fast}, and Noisy Natural
Gradient~\citep{zhang2018noisy} derive adaptive weight noise
from variational inference, with Vadam using running
squared-gradient statistics. MAGDrop uses clamped
absolute-gradient statistics as a direct noise-allocation rule.

\paragraph{Low-dimensional and partially stochastic posteriors.}
Last-layer Gaussian approximations~\citep{kristiadi2020being}
and subnetwork inference~\citep{daxberger2021subnetwork}
restrict stochasticity to selected weights for uncertainty
estimation. \citet{li2018intrinsic} study optimization in random
affine subspaces, and subspace inference~\citep{izmailov2020subspace}
performs Bayesian inference in random subspaces or subspaces
constructed from stochastic gradient descent (SGD) trajectories,
including PCA constructions. Stochastic weight
averaging--Gaussian~\citep{maddox2019simple}
fits a low-rank plus diagonal Gaussian covariance from SGD
iterates. Limitations of diagonal Gaussian posteriors for tight
PAC-Bayes bounds~\citep{pitas2020dissecting} further motivate
attention to posterior geometry.
MAGDrop-Subspace confines all stochastic variation to a fixed
$r$-dimensional affine subspace of the classifier head, yielding
an $r$-dimensional KL. The subspace basis, prior, and remaining
deterministic components are fixed using $S_0$, independently
of the certification sample $S$. The PCA variant of MAGDrop-Subspace uses gradients
on $S_0$ rather than weight iterates. Unlike low-rank weight
factorizations~\citep{toure2026singular}, $r$ denotes the dimension
of posterior variation, not the matrix rank of the sampled
classifier weights.

\paragraph{PAC-Bayes bounds for neural networks.}
PAC-Bayes theory~\citep{mcallester1999pac,catoni2007pac}
bounds randomized-predictor risk using empirical risk and KL
divergence from a prior independent of the certification sample.
\citet{dziugaite2017computing} obtained non-vacuous neural-network
bounds by optimizing a PAC-Bayes objective, and
\citet{perez-ortiz2021pac} obtained tight certificates using
PAC-Bayes training objectives and learned priors.
Sample-split prior learning~\citep{parrado2012pac,dziugaite2021role}
motivates our construction: Conditional on $S_0$, the prior and
subspace are fixed before certification on independent $S$.
\citet{lotfi2022compression} also exploit low-dimensional
subspaces, jointly optimizing quantized coordinates and
quantization levels to obtain tight compression-based bounds.
The compression-based complexity term is a description length, whereas
MAGDrop-Subspace uses a continuous Gaussian posterior with an
analytic KL. MAGDrop tracks gradient magnitudes to update
variance and certifies the resulting posterior after training.
Table~\ref{tab:cifar10} reports the architectures, sample splits,
confidence levels, and MC protocols alongside the
reference bounds.

\Needspace{3\baselineskip}
\paragraph{Curvature, perturbations, and certification.}
\citet{neyshabur2017exploring} connect expected sharpness under
weight perturbations with PAC-Bayes complexity, while
\citet{yang2022data} relate Hessian eigenspectra to effective
dimensionality and PAC-Bayes bounds. Gradient concentration in
sharp eigenspaces can depend on batch
normalization (BatchNorm)~\citep{ghorbani2019investigation}; the relevant
alignment is therefore model-dependent. These results
motivate examining posterior-noise orientation as well as
dimension. Our curvature analysis diagnoses the gap between
random QR and PCA; the numerical certificates use empirical Gibbs 0-1 risk
and analytic KL, without substituting a local quadratic
approximation into the bound.

\section{MAGDrop}
\label{sec:method}

MAGDrop uses gradient-momentum statistics to adapt a Gaussian
posterior over classifier weights. For certification, the two-phase
training procedure separates prior construction from posterior learning
and certification. Table~\ref{tab:notation} summarizes the
main notation.

\begin{table}[h]
\centering
\caption{Summary of Notation}
\label{tab:notation}
\renewcommand{\arraystretch}{1.20}
\begin{tabular}{@{}>{\raggedright\arraybackslash}p{0.32\columnwidth}@{\hspace{0.035\columnwidth}}>{\raggedright\arraybackslash}p{0.645\columnwidth}@{}}
\toprule
\textbf{Symbols} & \textbf{Description} \\
\midrule
$K$, $d$, $r$ & Output classes, input dimension, and subspace dimension \\
$\boldsymbol{\mu}_j$, $\boldsymbol{\mu}_j^{(0)}$ & Per-unit posterior and Phase~1 prior means in $\R^d$ \\
$W_0$, $U$ & Phase~1 head and orthonormal basis $U\in\R^{Kd\times r}$ \\
$\boldsymbol{\mu}_z\in\R^r$ & Latent posterior mean \\
$\eta_z=10^{-5}$ & Latent-mean learning rate \\
\midrule
$\sigma_j^2$, $\sigma_{z,i}^2$ & Posterior variances per head unit and latent direction \\
$\sigma_{\text{base}}^2=10^{-4}$ & Full-head base variance \\
$\sigma_{\text{prior}}^2$, $\tau^2$ & Full-head ($1/d$) and latent ($1/r$) prior variances \\
$\gamma_j\in[0.95,1.05]$ & Per-unit variance multiplier\newline under the default clamp range \\
$g_j^{(t)}$, $m_j^{(t)}$ & Per-unit mean absolute gradient and its EMA at step $t$ \\
$\beta=0.9$ & EMA momentum coefficient \\
\midrule
$Q$, $P$ & Posterior and prior distributions \\
$\KL(Q\|P)$ & KL divergence from $Q$ to $P$ \\
$L_{\mathcal{D}}(Q)$, $\widehat{L}_S(Q)$ & True and empirical Gibbs risks \\
$\widehat{L}_{S,M}^{\text{MC}}(Q)$ & MC estimate of $\widehat{L}_S(Q)$ \\
$n_0\coloneqq|S_0|$, $n\coloneqq|S|$ & Prior-construction and certification\newline sample sizes; $n_0=n=25000$ \\
$M=10$ & Posterior draws per example \\
$\delta\in(0,1]$ & Failure-probability parameter \\
$\lambda>0$ & Free parameter in PAC-Bayes proof \\
\bottomrule
\end{tabular}

\end{table}

Throughout this paper, $L_{\mathcal{D}}(Q)$ denotes the true
Gibbs risk, $\widehat{L}_S(Q)$ the empirical Gibbs risk on $S$,
and $\widehat{L}_{S,M}^{\text{MC}}(Q)$ the MC estimate of that
empirical risk. The empirical Gibbs risk averages over the
posterior exactly; the MC estimate uses finitely many draws.

\paragraph{Motivation.}
A standard linear classifier head computes
$\mathbf{y} \coloneqq W\mathbf{x} + \mathbf{b}$, where
$\mathbf{x}\in\R^d$ is the classifier input and
$W \in \R^{K \times d}$ is learned by gradient descent.
MAGDrop replaces uniform noise allocation with a per-unit
rule based on recent gradient activity, using an EMA to
smooth minibatch fluctuations.

\paragraph{Probabilistic formulation.}
MAGDrop replaces the deterministic weight matrix $W$
with a posterior distribution $Q$ over weights. For
each output unit $j \in \{1, \ldots, K\}$, the weight
vector $\mathbf{w}_j \in \R^d$ is modeled as an
isotropic Gaussian:
\begin{equation*}
    \mathbf{w}_j \sim \N(\boldsymbol{\mu}_j,
    \sigma_j^2 I), \qquad
    \sigma_j^2 \coloneqq \sigma_{\text{base}}^2 \cdot \gamma_j.
\end{equation*}
The mean $\boldsymbol{\mu}_j$ is the learnable weight
vector, trained by standard backpropagation exactly as
in a deterministic linear layer. The variance
$\sigma_j^2$ is not a fixed hyperparameter but a
dynamic quantity updated at every training step based
on observed gradient magnitudes. The multiplier
$\gamma_j \in [0.95, 1.05]$ encodes the gradient-based
adaptation, and $\sigma_{\text{base}}^2 = 10^{-4}$ is
a fixed base variance. Weights across different units and
input dimensions are assumed independent, making $Q$ a
fully factorized Gaussian, a standard and
computationally tractable posterior family for PAC-Bayes
analysis.
Here ``posterior'' has the PAC-Bayes meaning: A data-dependent
probability distribution~\citep{rivasplata2020beyond}.
The gradient-based update specifies this distribution directly.

\paragraph{Gradient-momentum adaptation.}
After each backward pass at training step $t$, we
compute the mean absolute gradient of the loss
$\mathcal{L}$ with respect to the sampled weights of
unit $j$:
\begin{equation*}
    g_j^{(t)} \coloneqq \frac{1}{d}\sum_{i=1}^{d}
    \left|\frac{\partial \mathcal{L}}
    {\partial \tilde{w}_{j,i}^{(t)}}\right|.
\end{equation*}
With $\beta=0.9$, we track recent activity using the EMA
\begin{equation*}
    m_j^{(t)} \coloneqq
    \beta \cdot m_j^{(t-1)} + (1 - \beta) \cdot g_j^{(t)}.
\end{equation*}
To make $\gamma_j$ comparable across units regardless
of the overall gradient scale, we normalize each unit's
momentum by the current mean momentum across all units.
Let $\bar{m}^{(t)} \coloneqq (1/K)\sum_{k=1}^{K}
m_k^{(t)}$. To handle the initialization case where
all momenta equal zero, we use the zero-safe convention:
\begin{equation}
\label{eq:momentum-normalization}
    \tilde{m}_j^{(t)} \coloneqq
    \begin{cases}
    m_j^{(t)}/\bar{m}^{(t)}, & \bar{m}^{(t)} > 0, \\
    1, & \bar{m}^{(t)} = 0.
    \end{cases}
\end{equation}
The fallback keeps $\gamma_j^{(t)}=1$ when all momenta
are zero. The subspace variant normalizes over $r$ latent
directions instead of $K$ head units. We update the multiplier as
\begin{equation}
\label{eq:gamma_update}
    \gamma_j^{(t)} \coloneqq
    \operatorname{clamp}(1-c(\tilde{m}_j^{(t)}-1),
    \gamma_{\min},\gamma_{\max}),
\end{equation}
where $c=0.30$, $\gamma_{\min}=0.95$, and $\gamma_{\max}=1.05$.

Above-average gradient activity gives $\gamma_j<1$ and
less noise; below-average activity gives $\gamma_j>1$ and
more noise. The clamp bounds the posterior variance
(Lemma~\ref{lem:stability}).

\paragraph{Forward pass and weight sampling.}
At each training step, a fresh weight matrix is sampled
from the current posterior:
\begin{equation*}
    \tilde{\mathbf{w}}_j \coloneqq \boldsymbol{\mu}_j
    + \sigma_{\text{base}} \sqrt{\gamma_j}
    \boldsymbol{\epsilon}_j, \qquad
    \boldsymbol{\epsilon}_j \sim \N(0, I).
\end{equation*}
The forward computation uses these sampled weights:
$\mathbf{y} \coloneqq \tilde{W}\mathbf{x} + \mathbf{b}$,
where $\tilde{W}$ has rows $\tilde{\mathbf{w}}_j^\top$.
The loss is computed from this stochastic output and
backpropagated in the usual way. A gradient hook
registered on $\tilde{W}$ intercepts the gradient
during the backward pass and uses the intercepted gradient to update $m_j$
and $\gamma_j$ before the optimizer step, without an additional
forward or backward pass.

\paragraph{Two-phase training.}
Training proceeds in two phases. In Phase~1, stochastic
sampling is disabled and the network is trained
deterministically for 30 epochs by default using Adam~\citep{kingma2015adam}
on $S_0$.
This deterministic phase allows the model to reach a meaningful initialization
before stochasticity is introduced. At the end of Phase~1,
the learned weights $\boldsymbol{\mu}^{(0)}$ are saved as
the mean of the prior distribution:
\begin{equation*}
    P: \quad \mathbf{w}_j \sim
    \N(\boldsymbol{\mu}_j^{(0)},
    \sigma_{\text{prior}}^2 I), \qquad
    \sigma_{\text{prior}}^2 \coloneqq \frac{1}{d}.
\end{equation*}
Centering the prior at $\boldsymbol{\mu}^{(0)}$ makes the
mean-drift contribution to KL depend on Phase~2 displacement
rather than the magnitude of the trained weights; KL also
includes variance mismatch.
This sample-splitting construction builds on
\citet{parrado2012pac,perez-ortiz2021pac} and the analysis of
learned priors by \citet{dziugaite2021role}; valid certification relies on excluding
$S_0$ from the empirical-risk term used for certification.

For certification, the backbone, bias, and stateful buffers
are frozen at the Phase~1 values. Phase~2 optimizes only
the head mean; the variance multipliers follow the
gradient-momentum update.
In Phase~2, stochastic sampling is enabled and training
continues for 100 additional epochs on $S$ by default. The gradient
momentum buffers are initialized to zero and the variance
multipliers to one at the start of Phase~2. After Phase~2, the posterior
parameters are frozen for evaluation. Reported accuracies use
the mean predictor, while certification averages 0-1 losses
over independent posterior draws, as detailed in
Sec.~\ref{sec:pacbayes}.

\paragraph{Inference.}
At test time, stochastic sampling is disabled and we
use the mean weights $\boldsymbol{\mu}_j$:
$\mathbf{y} \coloneqq \boldsymbol{\mu}\mathbf{x} + \mathbf{b}$.
This is a standard linear-layer computation with no additional
inference cost. Gibbs-risk certification instead uses posterior
sampling (Sec.~\ref{sec:pacbayes}).

\subsection{MAGDrop-Subspace: Low-Rank Posterior}
\label{sec:subspace}

\paragraph{The dimensionality problem.}
The full-head posterior incurs KL over $Kd$ stochastic
parameters. MAGDrop-Subspace restricts this cost to $r\ll Kd$
latent coordinates; the resulting complexity also depends on
posterior variance and learned mean drift.

\paragraph{Subspace basis construction.}
Random-subspace optimization~\citep{li2018intrinsic} and
Bayesian subspace inference~\citep{izmailov2020subspace}
motivate restricting weight uncertainty to a small affine space.
We construct the basis $U$ independently of $S$, preserving
the fixed-prior condition for certification. Throughout,
$\operatorname{vec}$ stacks the rows of $W$, and
$\operatorname{reshape}$ reverses this row-stacking operation. We investigate two options:

\paragraph{Gradient-trajectory PCA.}
At the end of each Phase~1 epoch $t$, we compute the
vectorized head gradient
\begin{equation*}
    \mathbf{g}_t \coloneqq
    \operatorname{vec}(\nabla_W\mathcal{L}(S_0; W_0^{(t)}))
    \in \R^{Kd},
\end{equation*}
using the current weights $W_0^{(t)}$ evaluated on all of $S_0$
without augmentation. Stacking $T$ such snapshots gives the
uncentered matrix $A \coloneqq [\mathbf{g}_1, \ldots, \mathbf{g}_T] \in
\R^{Kd \times T}$. Let $s \coloneqq \operatorname{rank}(A)
\leq \min(Kd, T)$; we take $U$ to be the top-$r$ left
singular vectors of $A$ when $r\leq s$. In this case, all directions
are fully data-driven; when $r > s$ the remaining
$r-s$ directions are filled with random orthogonal
vectors from the complement of the PCA span.

\paragraph{Random QR basis.}
$U$ is the orthonormal factor in a thin QR decomposition of
$B\in\R^{Kd\times r}$ with independent entries
$B_{ab}\sim\N(0,1)$, constructed
once using the $S_0$ random seed. This isotropic basis
places no structural prior on which
weight-space directions matter.

\paragraph{Subspace posterior.}
Let $W_0 \in \R^{K \times d}$ be the classifier weights
learned during Phase~1 on $S_0$, and let
$U \in \R^{Kd \times r}$ be the chosen orthonormal
basis, constructed by random QR by default. The posterior is defined
over a latent variable $\mathbf{z} \in \R^r$:
\begin{equation}
\label{eq:subspace_posterior}
\begin{split}
    W &\coloneqq W_0 + \operatorname{reshape}(U\mathbf{z}),\\
    \mathbf{z} &\sim \N(\boldsymbol{\mu}_z,
    \operatorname{diag}(\sigma_{z,1}^2, \ldots, \sigma_{z,r}^2)).
\end{split}
\end{equation}
The latent prior is
\begin{equation*}
    P:\quad \mathbf{z} \sim \N(0, \tau^2 I_r),
    \qquad \tau^2 \coloneqq \frac{1}{r}.
\end{equation*}
The posterior mean
$\boldsymbol{\mu}_z \in \R^r$ is the only
optimizer-updated parameter during Phase~2.
The variance multipliers $\boldsymbol{\gamma}$ are also
data-dependent parameters of the final posterior. These multipliers are
updated from Phase~2 gradients on $S$ and frozen at
certification, with the final multiplier values entering through
$\sigma_{z,i}^2 \coloneqq (\tau^2/2)\gamma_i$ in
\eqref{eq:subspace_kl}. All remaining deterministic
components, namely the backbone, $W_0$, $U$, the bias,
and $\tau^2$, are fixed using only $S_0$.
This construction satisfies the independence requirement.

\paragraph{KL divergence in $r$ dimensions only.}
The weight-space prior and posterior are the pushforwards of
the corresponding latent Gaussians through the same injective map
$\mathbf{z}\mapsto\operatorname{vec}(W_0)+U\mathbf{z}$.
Both weight-space distributions therefore share the same affine
support, and the weight-space KL equals the latent-coordinate KL. A lower-dimensional posterior
compared with a full-dimensional Gaussian prior would instead
have infinite KL, illustrating the support-mismatch
issue identified by \citet{hron2017variational}.
With the proper Gaussian prior and positive diagonal
posterior variances specified above, the KL factorizes into
$r$ scalar terms:
\begin{equation}
\label{eq:subspace_kl}
    \KL(Q \| P) = \frac{1}{2}\sum_{i=1}^{r}
    \left[
        \frac{\sigma_{z,i}^2}{\tau^2}
        + \frac{\mu_{z,i}^2}{\tau^2}
        - 1
        - \log\frac{\sigma_{z,i}^2}{\tau^2}
    \right].
\end{equation}
When the ratios of posterior to prior variance and the
average coordinate-wise mean-drift contributions are
held fixed, the variance-mismatch and mean-drift sums
scale linearly with dimension, suggesting an $O(Kd/r)$
dimensional advantage. The measured reduction also depends
on prior variances and learned mean drift, so dimension alone
does not determine the reduction factor.

\paragraph{Variance parameterization and KL decomposition.}
With subspace prior variance $\tau^2 \coloneqq 1/r$, the
posterior base variance and per-direction variances are
\begin{equation*}
    \sigma_{z,\text{base}}^2 \coloneqq \frac{\tau^2}{2} = \frac{1}{2r},\quad
    \sigma_{z,i}^2 \coloneqq \sigma_{z,\text{base}}^2 \gamma_i.
\end{equation*}
This base variance differs from the full-head base variance
$\sigma_{\text{base}}^2 = 10^{-4}$.

\Needspace{8\baselineskip}
With $\boldsymbol{\mu}_z = \mathbf{0}$ and
$\boldsymbol{\gamma}=\mathbf{1}_r$
at Phase~2 initialization, the KL reduces to the
pure variance-mismatch term:
\begin{equation}
\label{eq:kl_init}
\begin{split}
    \KL(Q\|P)_{\mathrm{init}}
    &= \frac{r}{2} \left[
        \frac{\sigma_{z,\text{base}}^2}{\tau^2}
        - 1
        - \log\frac{\sigma_{z,\text{base}}^2}{\tau^2}
    \right]\\
    &= \frac{r}{2}
       \left[\frac{1}{2} - 1 - \log\frac{1}{2}\right]
    \approx r \times 0.0966\text{ nats}.
\end{split}
\end{equation}
For $r=50$, the initial KL is approximately 4.83 nats.
Training changes both this variance-mismatch term and the
mean-drift term $\sum_i \mu_{z,i}^2/(2\tau^2)$.

\paragraph{Subspace adaptation.}
The gradient-momentum adaptation from
Sec.~\ref{sec:method} is applied to the $r$ latent
directions rather than the $K$ output units of the full-head variant.
Each stochastic forward pass draws:
\begin{equation}
\label{eq:subspace_sampling}
    \tilde{\mathbf{z}} \coloneqq \boldsymbol{\mu}_z
    + \frac{\tau}{\sqrt{2}}
    \sqrt{\boldsymbol{\gamma}}\odot\boldsymbol{\epsilon},\quad
    \boldsymbol{\epsilon}\sim\N(0,I_r),\quad
    \sigma_{z,i}^2=\frac{\tau^2}{2}\gamma_i,
\end{equation}
where the square root is elementwise and $\odot$ denotes
elementwise multiplication.
During the backward pass, the per-direction gradient
statistic is $g_i \coloneqq |\partial\mathcal{L}/\partial
\tilde{z}_i|$ for $i=1,\ldots,r$.
These statistics update the momentum buffer
$\mathbf{m} \in \R^r$ and the per-direction
multipliers $\boldsymbol{\gamma} \in \R^r$ via the
same rule as \eqref{eq:gamma_update}.

\paragraph{Learning rate for $\boldsymbol{\mu}_z$.}
The latent mean $\boldsymbol{\mu}_z$ is trained with a small
learning rate $\eta_z = 10^{-5}$, one hundredth
of the Phase~1 learning rate. Empirically,
with a standard learning rate, $\|\boldsymbol{\mu}_z\|$
grows substantially during Phase~2 training,
causing the mean-drift term in \eqref{eq:subspace_kl}
to grow to thousands of nats. The small learning
rate keeps $\|\boldsymbol{\mu}_z\|$ small throughout
training, preserving the tight KL.

\paragraph{Rank and basis selection.}
The primary random-QR configuration uses $r=50$, specified before
Phase~2 training using a held-out evaluation independent of
$S$. This choice balances the variance-mismatch cost in
\eqref{eq:kl_init} against posterior expressiveness.
Table~\ref{tab:basis} reports subsequent rank-and-basis
comparisons; their protocol differences and geometric
interpretation are analyzed in Sec.~\ref{sec:experiments}.

\begin{table}[t!]
\centering
\caption{Random QR and PCA Across Subspace Ranks}
\label{tab:basis}
\renewcommand{\arraystretch}{1.18}
\setlength{\tabcolsep}{0pt}
\begin{tabular}{@{}>{\raggedright\arraybackslash}p{0.12\columnwidth}
    S[table-format=1.4,table-column-width=0.13\columnwidth]
    S[table-format=2.2,table-column-width=0.12\columnwidth]
    S[table-format=1.4,table-column-width=0.18\columnwidth]
    @{\hspace{0.02\columnwidth}}
    S[table-format=1.4,table-column-width=0.13\columnwidth]
    S[table-format=2.2,table-column-width=0.12\columnwidth]
    S[table-format=1.4,table-column-width=0.18\columnwidth]@{}}
\toprule
& \multicolumn{3}{c}{\textbf{Random QR}}
& \multicolumn{3}{c}{\textbf{PCA}} \\
\cmidrule(lr){2-4} \cmidrule(lr){5-7}
\textbf{Rank} $r$
    & {\shortstack{\textbf{Gibbs}\\\textbf{risk}}}
    & {\shortstack{\textbf{KL}\\\textbf{(nats)}}}
    & {\shortstack{\textbf{PAC-Bayes-}\\\textbf{$\kl$ bound}}}
    & {\shortstack{\textbf{Gibbs}\\\textbf{risk}}}
    & {\shortstack{\textbf{KL}\\\textbf{(nats)}}}
    & {\shortstack{\textbf{PAC-Bayes-}\\\textbf{$\kl$ bound}}} \\
\midrule
$r=10$
    & \bfseries 0.1691 & 1.3 & \bfseries 0.1825
    & 0.2384 & \bfseries 1.03 & 0.2529 \\
$r=20$
    & \bfseries 0.1750 & 2.9 & \bfseries 0.1894
    & 0.1985 & \bfseries 2.19 & 0.2130 \\
$r=50$
    & \bfseries 0.1705 & 10.77 & \bfseries 0.1882
    & 0.1778 & \bfseries 7.17 & 0.1942 \\
\midrule
$r=100$
    & 0.1735 & 26.86 & 0.1969
    & \bfseries 0.1685 & \bfseries 22.04 & \bfseries 0.1902 \\
$r=200$
    & 0.1682 & 85.85 & 0.2050
    & \bfseries 0.1650 & \bfseries 67.83 & \bfseries 0.1980 \\
\bottomrule
\end{tabular}

\vspace{4pt}
\begin{minipage}{\linewidth}\footnotesize
Three-seed means; bold marks lower values within rank.
Random QR and PCA use 30/100 and 50/150 Phase~1/2 epochs,
respectively, with different anchors.
\end{minipage}
\end{table}

\begin{algorithm}[t!]
\caption{MAGDrop-Subspace: Training and Certification}
\label{alg:subspace}
\small
\begin{algorithmic}[1]
\REQUIRE Independent samples $S_0$ and $S$, with $S_0 \perp S$;
         rank $r$; $\eta_z = 10^{-5}$; $M$ posterior draws per example;
         $\delta_\text{PB}$, $\delta_\text{MC}$
\STATE Train the backbone $f$ and $W_0$ deterministically on $S_0$ in Phase~1
\STATE Construct $U \in \R^{Kd \times r}$ independently of $S$
       using random QR by default or gradient-trajectory PCA
       (Sec.~\ref{sec:subspace})
\STATE Freeze $f$, $W_0$, $U$, bias,
       and all stateful buffers, including the BatchNorm
       running mean and variance
       \COMMENT{No $S$-data reaches frozen components}
\STATE Set $\tau^2\leftarrow1/r$, $n\leftarrow|S|$;
       initialize $\boldsymbol{\mu}_z \leftarrow
       \mathbf{0}$, $\boldsymbol{\gamma} \leftarrow
       \mathbf{1}_r$, $\mathbf{m} \leftarrow
       \mathbf{0}$
\STATE Train on $S$ in Phase~2 for $T_2$ epochs;
       use $T_2=100$ for QR or $T_2=150$ for PCA by default
\FOR{$e = 1, \ldots, T_2$}
\FOR{each batch $(X, Y)$ in $S$}
    \STATE Sample $\tilde{\mathbf{z}} \leftarrow
           \boldsymbol{\mu}_z + (\tau/\sqrt{2})
           \sqrt{\boldsymbol{\gamma}}\odot \boldsymbol{\epsilon}$,
           $\boldsymbol{\epsilon} \sim \N(0, I_r)$
    \STATE $W \leftarrow W_0 +
           \operatorname{reshape}(U\tilde{\mathbf{z}})$
    \STATE Extract features $\Phi\leftarrow f(X)$;
           compute logits $W\Phi + \mathbf{b}$ and loss $\mathcal{L}$
    \STATE Backward; update $\mathbf{m}$,
           $\boldsymbol{\gamma}$ via
           \eqref{eq:gamma_update}
    \STATE Update $\boldsymbol{\mu}_z$ via Adam
           with learning rate $\eta_z$
\ENDFOR
\ENDFOR
\STATE Freeze $\boldsymbol{\mu}_z$, $\boldsymbol{\gamma}$,
       $W_0$, and $U$ after Phase~2; certify on $S$ using per-example MC
\STATE For each $(i,m)$: sample independently
       \par
       $\mathbf{z}_{i,m} \stackrel{\text{ind}}{\sim}
       \N(\boldsymbol{\mu}_{z,\text{final}},
       \operatorname{diag}((\tau^2/2)
       \boldsymbol{\gamma}_{\text{final}}))$
\STATE $W_{i,m} \leftarrow W_0 +
       \operatorname{reshape}(U\mathbf{z}_{i,m})$;
       $\mathbf{w}_{i,m}\leftarrow\operatorname{vec}(W_{i,m})$
\STATE $\widehat{L}_{S,M}^{\text{MC}}(Q) \leftarrow (1/(nM))
       \sum_{i=1}^n \sum_{m=1}^M
       \mathbf{1}[h_{\mathbf{w}_{i,m}}(\mathbf{x}_i) \neq y_i]$
\STATE $\varepsilon_\text{MC} \leftarrow
       \sqrt{\log(1/\delta_\text{MC})/(2nM)}$
       \COMMENT{Hoeffding over $nM$ conditionally independent losses}
\STATE $\varepsilon \leftarrow
       (\KL(Q\|P) +
       \log(2\sqrt{n}/\delta_\text{PB}))/n$
\RETURN
       $\kl^{-1}(\min\{1,\widehat{L}_{S,M}^{\text{MC}}(Q) +
       \varepsilon_\text{MC}\}, \varepsilon)$
\end{algorithmic}
\end{algorithm}

\Needspace{6\baselineskip}
\subsection{PAC-Bayes Generalization Bound}
\label{sec:pacbayes}

We state the assumptions for PAC-Bayes certification of
MAGDrop, establish a PAC-Bayes--Hoeffding bound, and derive
the PAC-Bayes-$\kl$ certificate used for all our numerical bounds.

\paragraph{Setup and independence requirement.}
Let $S_0 \sim \mathcal{D}^{n_0}$ be a
prior-construction sample and $S \sim
\mathcal{D}^n$ be an independent certification
sample, with $S_0 \perp S$.
Phase~1 training runs on $S_0$, producing the
data-dependent prior mean $\boldsymbol{\mu}^{(0)}$.
For the full-head construction, the prior is
$P:\mathbf{w}_j \sim \N(\boldsymbol{\mu}_j^{(0)},
\sigma_{\text{prior}}^2 I)$ with
$\sigma_{\text{prior}}^2 \coloneqq 1/d$.
MAGDrop-Subspace uses the latent prior in
Sec.~\ref{sec:subspace}, pushed forward onto the shared
affine support. The same independence argument applies.
Phase~2 training and all subsequent PAC-Bayes
computations use only the certification sample $S$.
Conditioning on $S_0$ leaves a fixed prior and an independent
sample $S$, as in sample-split analyses of learned
priors~\citep{dziugaite2021role}. The split ensures that prior
construction and empirical-risk evaluation satisfy the
independence condition of the bounds below.

\paragraph{Certified and deployment predictors.}
Write $\mathbf{w}\coloneqq\operatorname{vec}(W)$ for the classifier
weights and $h_{\mathbf{w}}$ for the resulting classifier,
including the frozen backbone and bias. The PAC-Bayes bound
below certifies the Gibbs predictor. Its true and empirical risks
average the loss over $\mathbf{w}\sim Q$:
\begin{align*}
    L_{\mathcal{D}}(Q) &\coloneqq
    \E_{\mathbf{w} \sim Q}
    \E_{(\mathbf{x},y)\sim\mathcal{D}}
    [\ell(h_{\mathbf{w}}(\mathbf{x}), y)],\\
    \widehat{L}_S(Q) &\coloneqq \frac{1}{n}
    \sum_{i=1}^n \E_{\mathbf{w}\sim Q}
    [\ell(h_{\mathbf{w}}(\mathbf{x}_i), y_i)].
\end{align*}
The Gibbs predictor samples $\mathbf{w} \sim Q$ independently
for each prediction. All our numerical bounds certify the Gibbs predictor's expected
0-1 risk through Eqs.~\ref{eq:kl_bound} and~\ref{eq:kl_inv}.

For deployment, the mean predictor uses $\mathbf{w}=\E_Q[\mathbf{w}]$,
\textit{i.e.}, the classifier with weight matrix
$\boldsymbol{\mu}$ for the full head or
$W_0+\operatorname{reshape}(U\boldsymbol{\mu}_z)$ for the
subspace variant, with zero additional inference overhead.
The mean predictor's risk $L_{\mathcal{D}}(\E_Q[\mathbf{w}])$ is a separate
quantity: Transferring a Gibbs-risk certificate to the mean predictor requires
a derandomization argument, such as a margin-based analysis
of prediction changes under weight
perturbations~\citep{biggs2022margins}.
The MC-logit ensemble is a third predictor that averages logits
from $M=10$ independent posterior draws and requires $M$
stochastic forward passes. All reported test and validation
accuracies are produced by the mean predictor
$h_{\E_Q[\mathbf{w}]}$, matching the zero-overhead deployment
setting. The PAC-Bayes-$\kl$ bounds certify the Gibbs predictor via the
per-example MC procedure of
Algorithm~\ref{alg:subspace}. Accuracy and bound thus characterize
different predictors derived from the same posterior $Q$.

\paragraph{Loss assumption.}
We require $0 \leq \ell \leq 1$. For multiclass
classification, we use the 0-1 loss
$\ell(h_\mathbf{w}(\mathbf{x}), y) \coloneqq
\mathbf{1}[h_\mathbf{w}(\mathbf{x}) \neq y]$,
which is bounded and satisfies the concentration
requirements. Cross-entropy loss is unbounded and
does not satisfy the Hoeffding condition without
additional truncation.

\Needspace{8\baselineskip}
\begin{lemma}[Variance Boundedness]
\label{lem:stability}
With the default clamp range, for any finite gradient trajectory
and using the normalization
in \eqref{eq:momentum-normalization}, the per-unit posterior
variance satisfies, for all units $j$ and training steps $t$,
\begin{equation*}
    0.95 \sigma_{\text{base}}^2 \leq \sigma_j^2
    \leq 1.05 \sigma_{\text{base}}^2.
\end{equation*}
\end{lemma}

\begin{proof}
The update rule applies a hard clamp so that
$\gamma_j \in [0.95, 1.05]$ at every step regardless
of the gradient magnitude or training dynamics. Since
$\sigma_j^2 = \sigma_{\text{base}}^2 \cdot \gamma_j$
and $\sigma_{\text{base}}^2 > 0$, multiplying the
bounds on $\gamma_j$ by $\sigma_{\text{base}}^2$
gives the stated inequalities immediately.
\end{proof}

\Needspace{5\baselineskip}
\begin{remark}
Because $m_j^{(t)}\ge0$, a zero mean occurs exactly when
every momentum is zero. The fallback $\tilde{m}_j^{(t)}=1$
in \eqref{eq:momentum-normalization} then gives
$\gamma_j^{(t)}=1$, so the variance update is well-defined
at initialization and throughout any finite gradient trajectory.
\end{remark}

\begin{theorem}[PAC-Bayes--Hoeffding Bound]
\label{thm:main}
Fix $\lambda > 0$ and $\delta \in (0, 1]$
independently of the certification sample $S$.
Under the independence condition $S_0 \perp S$, with
probability at least $1 - \delta$ over $S \sim
\mathcal{D}^n$, the following holds simultaneously
for all posteriors $Q$:
\begin{equation}
\label{eq:bound}
    L_{\mathcal{D}}(Q) \leq
    \widehat{L}_S(Q) +
    \frac{\KL(Q\|P) + \log(1/\delta)}{\lambda}
    + \frac{\lambda}{8n}.
\end{equation}
For the full-head Gaussian construction,
\begin{equation}
\label{eq:kl}
\begin{split}
    \KL(Q\|P) &= \frac{1}{2}\sum_{j=1}^{K}
    \sum_{i=1}^{d}\Biggl(
    \frac{\sigma_j^2 + (\mu_{j,i} -
    \mu_{j,i}^{(0)})^2}{\sigma_{\text{prior}}^2}\\
    &\qquad\qquad - 1 - \log\frac{\sigma_j^2}
    {\sigma_{\text{prior}}^2}\Biggr).
\end{split}
\end{equation}
The subspace construction uses \eqref{eq:subspace_kl} instead.
For $n\geq8$, the PAC-Bayes-$\kl$ inequality in
\eqref{eq:kl_bound}, together with Pinsker's
inequality, yields the following separate
square-root bound, also simultaneous over all $Q$ at confidence
$1-\delta$:
\begin{equation}
\label{eq:bound_opt}
    L_{\mathcal{D}}(Q) \leq
    \widehat{L}_S(Q) +
    \sqrt{\frac{\KL(Q\|P) +
    \log(2\sqrt{n}/\delta)}{2n}}.
\end{equation}
\end{theorem}

\begin{proof}
For a fixed $\mathbf{w}$, write
$L_{\mathcal{D}}(\mathbf{w})\coloneqq
\E_{(\mathbf{x},y)\sim\mathcal{D}}\ell(h_\mathbf{w}(\mathbf{x}),y)$ and
$\widehat{L}_S(\mathbf{w})\coloneqq
n^{-1}\sum_i\ell(h_\mathbf{w}(\mathbf{x}_i),y_i)$.
We bound $L_{\mathcal{D}}(\mathbf{w}) -
\widehat{L}_S(\mathbf{w})$.
Since the loss is bounded
in $[0,1]$, Hoeffding's lemma~\citep{hoeffding1963probability}
applied to $n$
independent and identically distributed samples gives, for any fixed $\mathbf{w}$
and $\lambda > 0$:
\begin{equation*}
    \E_S \left[e^{\lambda(L_{\mathcal{D}}(\mathbf{w}) -
    \widehat{L}_S(\mathbf{w}))}\right]
    \leq e^{\lambda^2/(8n)}.
\end{equation*}
By Fubini's theorem and the independence
$S_0 \perp S$, this exponential-moment bound extends to expectations over
$\mathbf{w} \sim P$. Applying Markov's inequality
to the prior exponential moment gives the following event:
With probability at least $1-\delta$ over $S$,
\begin{equation*}
    \E_{\mathbf{w}\sim P}
     \left[e^{\lambda(L_{\mathcal{D}}(\mathbf{w}) -
    \widehat{L}_S(\mathbf{w}))}\right]
    \leq \frac{1}{\delta} e^{\lambda^2/(8n)}.
\end{equation*}
On this event, the Donsker--Varadhan variational
formula~\citep{donsker1975asymptotic}, also presented in
\citet[Appendix~A]{seeger2002pac}, implies that for every
posterior $Q$:
\begin{equation*}
    \lambda(L_{\mathcal{D}}(Q) -
    \widehat{L}_S(Q))
    \leq \KL(Q\|P) + \log\frac{1}{\delta}
    + \frac{\lambda^2}{8n}.
\end{equation*}
Dividing by $\lambda$ gives
\eqref{eq:bound}. \eqref{eq:bound_opt}
follows from \eqref{eq:kl_bound} and Pinsker's
inequality $2(q-\hat{p})^2 \leq \kl(\hat{p}\|q)$.
\end{proof}

\paragraph{PAC-Bayes-$\kl$ certificate.}
For numerical certificates, we use the tighter and
$\lambda$-free PAC-Bayes-$\kl$ bound. For $n \geq 8$,
the Langford--Seeger--Maurer
inequality~\citep{seeger2002pac,maurer2004pac}
states that under the same independence condition, with
probability at least $1-\delta$ over
$S\sim\mathcal{D}^n$, simultaneously for all
posteriors $Q$:
\begin{equation}
\label{eq:kl_bound}
    \kl(\widehat{L}_S(Q)\|L_{\mathcal{D}}(Q))
    \leq
    \frac{\KL(Q\|P) +
    \log \left(\frac{2\sqrt{n}}{\delta}\right)}{n},
\end{equation}
where the binary KL divergence is
\begin{equation*}
    \kl(a\|b) \coloneqq a\log\frac{a}{b}
    + (1-a)\log\frac{1-a}{1-b}.
\end{equation*}
Consequently, inverting the upper
branch of the binary-KL constraint gives the
following valid risk upper bound:
\begin{equation}
\label{eq:kl_inv}
    L_{\mathcal{D}}(Q) \leq
    \kl^{-1} \left(\widehat{L}_S(Q),
    \frac{\KL(Q\|P) +
    \log \left(\frac{2\sqrt{n}}{\delta}\right)}{n}
    \right),
\end{equation}
where $\kl^{-1}(\hat{p}, \varepsilon)$ denotes the
largest $q \in [\hat{p}, 1]$ satisfying
$\kl(\hat{p}\|q) \leq \varepsilon$.

\paragraph{MC empirical risk and confidence
correction.}

In practice, $\widehat{L}_S(Q)$ is estimated by
per-example independent sampling: For each
example $(\mathbf{x}_i,y_i) \in S$ and each replicate
$m\in\{1,\ldots,M\}$, an independent weight vector
$\mathbf{w}_{i,m} \sim Q$ is drawn. The estimator is
\begin{equation*}
    \widehat{L}_{S,M}^{\text{MC}}(Q) \coloneqq
    \frac{1}{nM}\sum_{i=1}^n\sum_{m=1}^M
    \ell(h_{\mathbf{w}_{i,m}}(\mathbf{x}_i), y_i).
\end{equation*}
The $nM$ losses
$\ell(h_{\mathbf{w}_{i,m}}(\mathbf{x}_i),y_i)$
are independent random variables bounded in $[0,1]$
conditional on the fixed certification sample $S$,
the fixed posterior $Q$, and all pre-certification
randomness. These losses need not be identically distributed
because the examples $\mathbf{x}_i$ differ. By
Hoeffding's inequality~\citep{hoeffding1963probability} over
all $nM$ losses, with
probability at least $1-\delta_{\text{MC}}$:
\begin{equation*}
    \widehat{L}_S(Q) \leq
    \widehat{L}_{S,M}^{\text{MC}}(Q)
    + \sqrt{\frac{\log(1/\delta_{\text{MC}})}{2nM}}.
\end{equation*}
Define the upper confidence bound
\begin{equation*}
\widehat{L}_{S,M}^+(Q)
\coloneqq \min\left\{1,
\widehat{L}_{S,M}^{\text{MC}}(Q)
+\sqrt{\frac{\log(1/\delta_{\text{MC}})}{2nM}}\right\}.
\end{equation*}
Substitution into \eqref{eq:kl_inv} yields a certificate
with total confidence at least
$1 - \delta_{\text{PB}} - \delta_{\text{MC}}$.
For $n=25000$, $M=10$, and $\delta_{\text{MC}}=0.05$, the correction is
$\sqrt{\log(20)/(2\times250000)} \approx 0.00245$,
which is small but is included in every reported bound.

\paragraph{KL divergence computation.}
The total KL in \eqref{eq:kl} sums over all
$K \times d$ stochastic weight components in the
MAGDrop head. For CIFAR-10, $K=10$, $d=600$,
$\sigma_{\text{base}}^2 = 10^{-4}$, and
$\sigma_{\text{prior}}^2 = 1/600$. The variance
mismatch term contributes approximately 5,620 nats
and the mean drift during Phase~2 contributes the
remainder. Full-head KL values are computed using
\eqref{eq:kl}; subspace KL values are
computed using \eqref{eq:subspace_kl}, which
sums over the $r$ latent dimensions only.

\begin{proposition}[Comparison with the Pinsker relaxation]
\label{prop:tight}
Let $\varepsilon \coloneqq [\KL(Q\|P) +
\log(2\sqrt{n}/\delta)]/n$. Then
\begin{equation*}
    \kl^{-1}(\hat{p}, \varepsilon)
    \leq \min \left\{1, \hat{p} +
    \sqrt{\varepsilon/2}\right\}.
\end{equation*}
Hence \eqref{eq:kl_inv} is no looser than the corresponding
Pinsker relaxation using the same complexity term.
\end{proposition}
\begin{proof}
Set $q \coloneqq \kl^{-1}(\hat{p}, \varepsilon)$. Pinsker's
inequality gives $2(q-\hat{p})^2 \leq
\kl(\hat{p}\|q) \leq \varepsilon$, so
$q \leq \hat{p} + \sqrt{\varepsilon/2}$.
\end{proof}

\section{Experiments}
\label{sec:experiments}

\paragraph{Training protocol.}
CIFAR-10~\citep{krizhevsky2009learning} certification
experiments follow Algorithm~\ref{alg:subspace}, using the
default schedule of 30 Phase~1 and 100 Phase~2 epochs.
Phase~1 and full-head Phase~2 runs use Adam~\citep{kingma2015adam} with
learning rate $10^{-3}$ and weight decay $5\times10^{-4}$.
For MAGDrop-Subspace, Phase~2 trains
$\boldsymbol{\mu}_z$ using Adam with learning rate $10^{-5}$
and weight decay $5\times10^{-4}$.
For the subspace basis ablation (Table~\ref{tab:basis}),
the PCA variant uses 50 epochs in Phase~1 and 150 in Phase~2
to ensure sufficient gradient snapshots.
Unless stated otherwise, results are three-seed means,
reported uncertainties are standard deviations (SDs), and bold
marks the best result.

\paragraph{Baseline construction.}
The full-head baseline follows the freezing rules in
Algorithm~\ref{alg:subspace}, including any EMA buffers for
model weights, but optimizes the head mean $\boldsymbol{\mu}$
on $S$ with deterministic forward passes. Certification uses
$Q_{\text{base}} \coloneqq \N(\operatorname{vec}(\boldsymbol{\mu}),
\sigma_{\text{base}}^2 I)$ with $\gamma_j\equiv1$ and the same
per-example MC procedure as MAGDrop-Subspace. Test accuracy
uses the mean predictor.

\paragraph{Wall-clock training overhead.}
Table~\ref{tab:timing} reports timings on a Tesla T4 GPU in
Kaggle's two-T4 environment. The 2.5\% training-time decrease
is within measurement noise: the backbone forward-backward
pass is still executed with the frozen backbone, dominating
the cost of updating the latent parameters. The test-set MC
benchmark measures stochastic evaluation cost; certification
instead follows Algorithm~\ref{alg:subspace} on $S$. Data loading
limits the observed timing ratio, while pure forward-pass
cost scales with $M$.

\begin{table}[t!]
\centering
\caption{CIFAR-10 Timing for a 13-Layer CNN on Tesla T4}
\label{tab:timing}
\renewcommand{\arraystretch}{1.12}
\begin{tabular*}{\columnwidth}{@{\extracolsep{\fill}}lrr@{}}
\toprule
\textbf{Configuration} & \textbf{Time (s)} & \textbf{Change (\%)} \\
\midrule
\multicolumn{3}{@{}l}{\textit{Training per epoch}} \\
\addlinespace[1pt]
Phase~1, deterministic & 2.41 & -- \\
Phase~2, subspace, $r=50$ & 2.35 & $-2.5$ \\
Phase~2, baseline, 6,000 parameters & 2.43 & $+0.8$ \\
\midrule
\multicolumn{3}{@{}l}{\textit{Inference on the full test set}} \\
\addlinespace[1pt]
Deterministic, one pass & 1.55 & $0.0$ \\
MC sampling, $M=10$ & 3.80 & $+145.2$ \\
\bottomrule
\end{tabular*}

\vspace{4pt}
\begin{minipage}{\linewidth}\footnotesize
Training: five-epoch means on 10,000 examples; inference:
three-run means on the full 10,000-image test set.
Changes are relative to Phase~1 and deterministic inference,
respectively.
\end{minipage}

\end{table}

\paragraph{PAC-Bayes bound computation.}
All certificates use 0-1 loss and \eqref{eq:kl_inv} with
$\delta_{\text{PB}}=\delta_{\text{MC}}=0.05$, giving total
confidence 0.90. Algorithm~\ref{alg:subspace} uses $M=10$
independent posterior draws per example and includes the
Hoeffding MC correction, approximately 0.00245 for $n=25000$.
The 50,000 CIFAR-10 training examples are split into independent
samples $S_0$ and $S$ with $n_0=n=25000$. Full-head and
subspace constructions share the Phase~1 anchor,
$\boldsymbol{\mu}^{(0)}=W_0$, but use priors on their respective
supports, as defined in Sec.~\ref{sec:method}.

\paragraph{CIFAR-10 with a 13-layer CNN.}
These experiments use two T4 graphics processing units (GPUs)
at Kaggle.
We evaluate both the full-head MAGDrop and the proposed
MAGDrop-Subspace variant on a 13-layer CNN with a
600-dimensional classifier input. The 13-layer CNN of
\citet[Table~5]{perez-ortiz2021pac} instead uses 512
classifier inputs and stochastic weights throughout the network.
Their certificates use $\delta=0.025$, $\delta'=0.01$,
150,000 MC model samples, and 25,000 or 15,000 certification
examples for 50\% or 70\% prior data, respectively.
Table~\ref{tab:cifar10} summarizes the certificates;
Table~\ref{tab:basis} reports the
basis comparison. The results show how posterior dimensionality,
basis choice, and variance adaptation affect certification.

\paragraph{Full-head posteriors yield loose bounds.}
The full-head baseline accumulates
KL $\approx$ 8,625 nats. For the observed
empirical Gibbs risk of approximately 0.07, this KL
yields a loose PAC-Bayes-$\kl$ bound of 0.4439
(Table~\ref{tab:cifar10}).
Posterior complexity limits certification for this full-head
parameterization despite its low empirical Gibbs risk.

\paragraph{MAGDrop-Subspace achieves non-vacuous bounds.}
With the random QR basis at $r=50$, KL is reduced to
$\approx$10.77 nats, an approximately 800$\times$
reduction from the full-head baseline KL of 8,625 nats.
The variance-mismatch and mean-drift terms contribute
approximately 4.83 and 5.94 nats, respectively; the final
multipliers have an SD of approximately 0.04.
The bound improves to 0.1882 $\pm$ 0.0015,
substantially tighter than the full-head baseline bound of 0.4439,
as shown in Table~\ref{tab:cifar10} and Figure~\ref{fig:cifar10}(a)--(b).
Each seed produces a valid 90\%-confidence certificate:
$B_1=0.1883$ for seed~42, $B_2=0.1863$ for seed~43,
and $B_3=0.1899$ for seed~44; the reported mean 0.1882 is a
descriptive summary of these three individually valid
certificates. Across three seeds, the baseline achieves
83.74\% $\pm$ 0.34\% and MAGDrop-Subspace achieves
82.31\% $\pm$ 0.18\%. The observed 1.43 percentage-point
gap quantifies the accuracy cost of the tighter certificate.

\paragraph{Random basis outperforms PCA at $r \leq 50$.}
Table~\ref{tab:basis} favors random QR at low rank despite
PCA's lower KL, because PCA has higher Gibbs risk.
The rank-$10$ random-QR bound is marginally below the
pre-specified rank-$50$ result, with a gap comparable to seed
variation (Table~\ref{tab:cifar10}). At $r\geq100$, the
PCA--random hybrid achieves lower KL and Gibbs risk under its
longer training schedule. The controlled analysis below
(Figures~\ref{fig:qual_basis_geometry} and~\ref{fig:qual_basis_controlled})
examines posterior sensitivity and protocol dependence.

\paragraph{Variance ablation at $r=50$.}
The fixed control uses $\gamma_i\equiv1$. The random control
fixes independent $\mathrm{Uniform}[0.95,1.05]$ draws at
Phase~2 initialization; the shuffled control applies a fixed
permutation to the adaptive vector after each epoch.
The random draws and permutation use the same $S$-independent
seed.

At $r=50$, MAGDrop-Subspace has a bound of 0.1882, tighter than the
0.1938 obtained with fixed $\gamma$ but looser than the 0.1858
and 0.1837 obtained with random and shuffled $\gamma$,
respectively, as shown in Table~\ref{tab:cifar10} and
Figure~\ref{fig:cifar10}(c). All four variants lie within about 0.010 of one
another, comparable to the seed-to-seed SDs
of 0.002 to 0.007. The largest certificate improvement comes
from the low-rank subspace structure. Within that subspace,
adaptive variance improves on fixed variance but shows no
consistent advantage over the random and shuffled controls
in these three-seed experiments.

\begin{table*}[t!]
\centering
\caption{CIFAR-10 Certificates for 13-Layer CNNs}
\label{tab:cifar10}
\renewcommand{\arraystretch}{1.18}
\begin{tabular*}{\textwidth}{@{\extracolsep{\fill}}ll
    S[table-format={\approx1.4}] S[table-format={\approx4.1{$^\ddagger$}}]
    S[table-format={\approx1.4(4)}]@{}}
\toprule
\textbf{Method} & \textbf{Configuration}
    & {\textbf{Gibbs risk}} & {\textbf{KL (nats)}}
    & {\textbf{PAC-Bayes-$\kl$ bound}} \\
\midrule
\multicolumn{5}{@{}l}{\textit{External references: 13-layer CNN}, \citet[Table~5]{perez-ortiz2021pac}} \\
\addlinespace[1pt]
$f_{\mathrm{quad}}$ & 50\% prior data & {--} & {--} & 0.2256 \\
\multirow{2}{*}{$f_{\mathrm{classic}}$}
 & 50\% prior data & {--} & {--} & 0.2127 \\
 & 70\% prior data & {--} & {--} & 0.1758 \\
\midrule
\multicolumn{5}{@{}l}{\textit{Full-head stochastic posteriors with large KL}} \\
\addlinespace[1pt]
Full-head baseline & Fixed $\gamma$ & 0.0692 & 8625 & 0.4439 +- 0.010 \\
Full-head MAGDrop & Adaptive $\gamma$$^\dagger$
 & \approx 0.07 & \approx 8625 & \approx 0.44 \\
\midrule
\multicolumn{5}{@{}l}{\textit{MAGDrop-Subspace rank sweep}} \\
\addlinespace[1pt]
\multirow{5}{*}{MAGDrop-Subspace}
 & $r=10$ & 0.1691 & 1.3 & 0.1825 +- 0.005 \\
 & $r=20$ & 0.1750 & 2.9 & 0.1894 +- 0.007 \\
 & $r=50$ & 0.1705 & 10.8{$^\ddagger$}
 & \bfseries 0.1882 +- 0.0015 \\
 & $r=100$ & 0.1735 & 26.9 & 0.1969 +- 0.007 \\
 & $r=200$ & 0.1682 & 85.9 & 0.2050 +- 0.007 \\
\midrule
\multicolumn{5}{@{}l}{\textit{Variance ablation at $r=50$}} \\
\addlinespace[1pt]
\multirow{3}{*}{Subspace}
 & Fixed $\gamma$ & 0.1756 & 11.5 & 0.1938 +- 0.007 \\
 & Random $\gamma$ & 0.1678 & 11.7 & 0.1858 +- 0.003 \\
 & Shuffled $\gamma$ & 0.1658 & 11.6 & 0.1837 +- 0.003 \\
\bottomrule
\end{tabular*}

\vspace{4pt}
\begin{minipage}{\linewidth}\footnotesize
Our results: three-seed means; $\pm$ denotes SD. Bold marks
the pre-specified $r=50$ configuration. Reference architectures,
posteriors, and certification protocols differ
(Sec.~\ref{sec:experiments}).

\smallskip
$\dagger$ The clamp changes full-head KL by $<5\%$ relative
to the baseline.
\par
$\ddagger$ Rounded from 10.77 nats; bounds use unrounded KL.
\end{minipage}

\end{table*}

\begin{figure*}[t!]
\centering
\includegraphics[width=\textwidth]{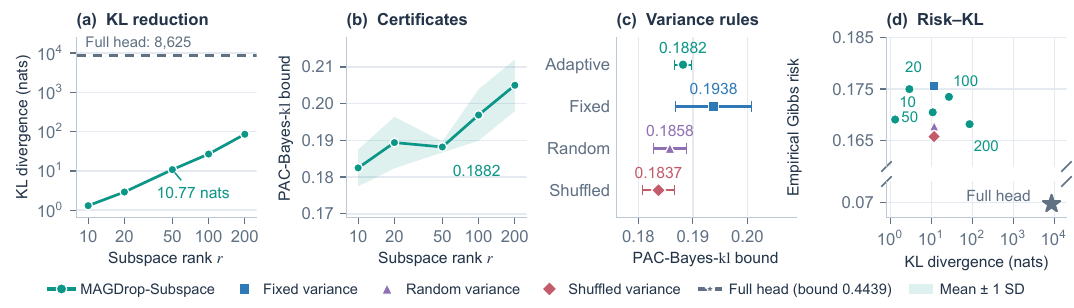}
\caption{Random-QR MAGDrop-Subspace on CIFAR-10 with a 13-layer CNN.
(a) KL and (b) PAC-Bayes-$\kl$ bounds versus rank;
(c) variance ablations at $r=50$; (d) risk--KL frontier
with rank labels and a broken risk axis.
Bands and error bars show three-seed means $\pm$ SD.
The dashed line and star denote the full-head baseline.}

\label{fig:cifar10}
\end{figure*}

\paragraph{Hyperparameter sensitivity analysis.}
Table~\ref{tab:ablation} examines the variance multiplier
clamp range, EMA momentum $\beta$, and adaptation coefficient
$c$. All experiments use the subspace protocol with $r=50$,
a random QR basis, 30 epochs in Phase~1, 100 in Phase~2,
$S_0\perp S$, and three seeds distinct from the seeds in
Table~\ref{tab:cifar10}. The default configuration yields a
bound of 0.1829, compared with 0.1882 in
Table~\ref{tab:cifar10}, reflecting variation between seed sets.
Deltas are relative to the default in this analysis.
Only one hyperparameter varies at a time; the remaining hyperparameters are fixed
at the default clamp range $[0.95,1.05]$, $\beta=0.90$,
and $c=0.30$.

The method is robust across the tested hyperparameters.
All seven configurations yield non-vacuous bounds in
$[0.1758,0.1829]$, a maximum spread of 0.0071
(Table~\ref{tab:ablation}). Every
alternative achieves a comparable or tighter certificate
than the default in this analysis.

The settings $c=0.10$ and clamp range $[0.99,1.01]$ have
the lowest mean bounds in the coefficient and clamp groups, respectively. Setting
$c=0.50$ also improves on $c=0.30$, and both wider and tighter
clamps improve on the default, giving a nonmonotonic response
across the tested settings. The primary experiments use the
default settings specified above.

\begin{table*}[t!]
\centering
\caption{MAGDrop-Subspace Hyperparameter Sensitivity at $r=50$ with Random QR}
\label{tab:ablation}
\renewcommand{\arraystretch}{1.18}
\begin{tabular*}{\textwidth}{@{\extracolsep{\fill}}ll
    S[table-format=1.4] S[table-format=1.1]
    S[table-format=1.4(4)] S[table-format=-1.4]@{}}
\toprule
\textbf{Group} & \textbf{Configuration}
      & {\textbf{Gibbs risk}}
      & {\textbf{KL (nats)}}
      & {\textbf{PAC-Bayes-$\kl$ bound}}
      & {$\Delta$} \\
\midrule
\multirow{3}{*}{A: Clamp}
  & Wider, $[0.90,1.10]$
      & 0.1608 & 9.0 & 0.1775 +- 0.0038 & -0.0054 \\
  & Default$^\dagger$, $[0.95,1.05]$
      & 0.1661 & 8.9 & 0.1829 +- 0.0031 & 0.0000 \\
  & Tighter, $[0.99,1.01]$
      & 0.1601 & 8.8 & 0.1767 +- 0.0057 & -0.0062 \\
\midrule
\multirow{2}{*}{B: Momentum}
  & Fast, $\beta=0.70$
      & 0.1628 & 8.5 & 0.1793 +- 0.0035 & -0.0036 \\
  & Slow, $\beta=0.95$
      & 0.1618 & 9.2 & 0.1786 +- 0.0081 & -0.0043 \\
\midrule
\multirow{2}{*}{C: Coefficient}
  & Gentle, $c=0.10$
      & \bfseries 0.1596 & 8.0 & \bfseries 0.1758 +- 0.0027 & -0.0071 \\
  & Aggressive, $c=0.50$
      & 0.1608 & 8.2 & 0.1771 +- 0.0084 & -0.0058 \\
\bottomrule
\end{tabular*}

\vspace{4pt}
\begin{minipage}{\linewidth}\footnotesize
$\dagger$ Defaults: clamp $[0.95,1.05]$, $\beta=0.90$, $c=0.30$.
$\Delta$: bound change from default; seeds differ from
Table~\ref{tab:cifar10}.
\end{minipage}
\end{table*}

\paragraph{Cross-architecture certification on ResNet-20.}
To evaluate whether the certificate generalizes beyond the
13-layer CNN, we repeat the main CIFAR-10 experiment on
a 20-layer residual network (ResNet-20)~\citep{he2016deep}
using two T4 GPUs at Kaggle.
This residual architecture has a substantially different
inductive bias. The protocol is identical: We split the
training set into independent samples $S_0$ and $S$ of
25,000 examples each, use a random QR basis with $r=50$,
train for 30 epochs in Phase~1 and 100 in Phase~2, and
certify with $M=10$ per-example MC samples and
$\delta_{\text{PB}}=\delta_{\text{MC}}=0.05$ over three seeds.
The classifier head has $Kd=64\times10=640$ parameters,
so the subspace covers $r/(Kd)\approx7.8\%$ of the classifier's weight space.

Table~\ref{tab:resnet20} reports a non-vacuous bound of
$0.1996\pm0.0107$, close to the 0.1882 obtained on the
13-layer CNN (Table~\ref{tab:cifar10}) and supporting the portability of the
low-dimensional posterior. Accuracy is 82.00\%, compared
with 84.43\% for the deterministic baseline, a decrease
of approximately 2.4 percentage points.
The fixed-variance ablation gives a nearly identical bound
of 0.2001. The difference of 0.0005 is small relative to
the SDs of 0.0107 and 0.0035 for MAGDrop-Subspace
and fixed variance, respectively. Both variance rules therefore
yield similar certificates with the same low-rank construction.

\begin{table}[tbp]
\centering
\caption{ResNet-20 Certificates on CIFAR-10 with Random QR, $r=50$}
\label{tab:resnet20}
\renewcommand{\arraystretch}{1.18}
\setlength{\tabcolsep}{3pt}
\begin{tabular*}{\columnwidth}{@{\extracolsep{\fill}}l
    S[table-format=2.2] S[table-format=1.4]
    S[table-format=1.2] S[table-format=1.4(4)]@{}}
\toprule
\textbf{Method} & {\shortstack{\textbf{Accuracy}\\\textbf{(\%)}}}
    & {\shortstack{\textbf{Gibbs}\\\textbf{risk}}}
    & {\shortstack{\textbf{KL}\\\textbf{(nats)}}}
    & {\shortstack{\textbf{PAC-Bayes-$\kl$}\\\textbf{bound}}} \\
\midrule
\shortstack[l]{ResNet-20\\deterministic}
    & \bfseries 84.43 & {--} & {--} & {--} \\
\midrule
\shortstack[l]{MAGDrop-\\Subspace}
    & 82.00 & 0.1833 & 6.56
    & \bfseries 0.1996 +- 0.0107 \\
\shortstack[l]{Subspace,\\fixed $\gamma$}
    & 82.02 & 0.1837 & 6.73
    & 0.2001 +- 0.0035 \\
\bottomrule
\end{tabular*}
\end{table}

\paragraph{ImageNet-1K with ViT-Base/16 and ConvNeXt-Base.}
We evaluate the base Vision Transformer with $16\times16$
patches (ViT-Base/16)~\citep{dosovitskiy2021image} and
ConvNeXt-Base~\citep{liu2022convnet} using four A100 GPUs
on Polaris at the Argonne Leadership Computing Facility (ALCF). Both models are trained from scratch for 300 epochs on
ImageNet-1K~\citep{russakovsky2015imagenet} with
Adam with decoupled weight decay~\citep{loshchilov2019decoupled},
a cosine learning rate schedule,
mixup~\citep{zhang2018mixup} with $\alpha=0.8$,
CutMix~\citep{yun2019cutmix} with $\alpha=1.0$,
RandAugment~\citep{cubuk2020randaugment}, and an EMA of model
weights. This model EMA is distinct from the gradient-momentum
EMA used to adapt MAGDrop's variance. ImageNet-1K accuracies
use the mean predictor with EMA model weights.
Table~\ref{tab:imagenet} and Figure~\ref{fig:convnext_curves}
summarize results.

For ViT-Base/16, MAGDrop achieves 78.57\% accuracy compared
with 78.51\% for the strong baseline. For ConvNeXt-Base,
MAGDrop achieves 82.14\% compared with 82.23\%
(Table~\ref{tab:imagenet}). These single-run
differences are $+0.06$ and $-0.09$ percentage points,
respectively. Figure~\ref{fig:convnext_curves}
shows virtually identical ConvNeXt-Base training trajectories;
Figure~\ref{fig:convnext_curves}(c) reports the best point
estimates for both architectures.

The ImageNet-1K experiments assess MAGDrop's compatibility
with strong training recipes on large-scale architectures.
The complementary CIFAR-10 experiments provide the PAC-Bayes
certificates reported in Tables~\ref{tab:cifar10} and \ref{tab:resnet20}.

\begin{table}[t!]
\centering
\caption{ImageNet-1K Top-1 Accuracy}
\label{tab:imagenet}
\renewcommand{\arraystretch}{1.12}
\begin{tabular*}{\columnwidth}{@{\extracolsep{\fill}}llr@{}}
\toprule
\textbf{Model} & \textbf{Method} & \textbf{Top-1 accuracy (\%)} \\
\midrule
\multirow{2}{*}{ViT-Base/16}
 & Baseline & 78.51 \\
 & MAGDrop & \textbf{78.57} \\
\midrule
\multirow{2}{*}{ConvNeXt-Base}
 & Baseline & \textbf{82.23} \\
 & MAGDrop & 82.14 \\
\bottomrule
\end{tabular*}

\vspace{4pt}
\begin{minipage}{\linewidth}\footnotesize
Best accuracy from single 300-epoch runs with EMA model weights.
\end{minipage}
\end{table}

\begin{figure*}[t!]
\centering
\includegraphics[width=\textwidth]{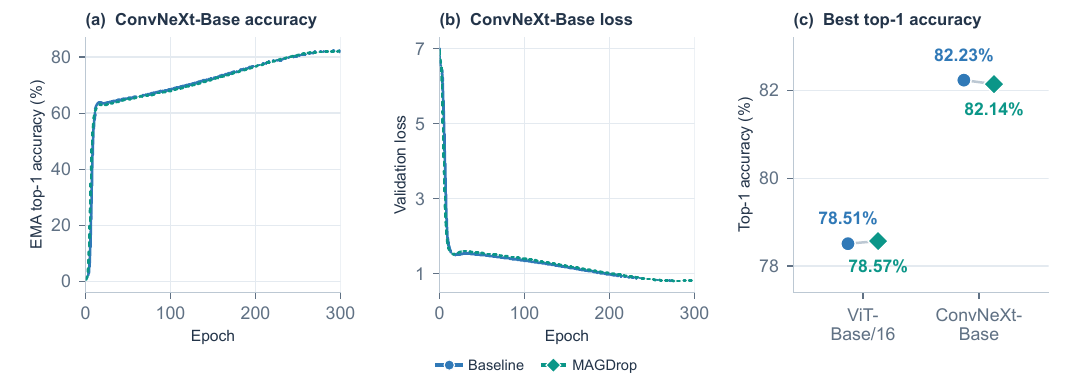}
\caption{ImageNet-1K results from single 300-epoch runs using
EMA model weights. (a)--(b) ConvNeXt-Base accuracy and validation
loss digitized from training plots; (c) best top-1 accuracy.}

\label{fig:convnext_curves}
\end{figure*}

\begin{figure*}[tp]
\centering
\includegraphics[width=\textwidth]{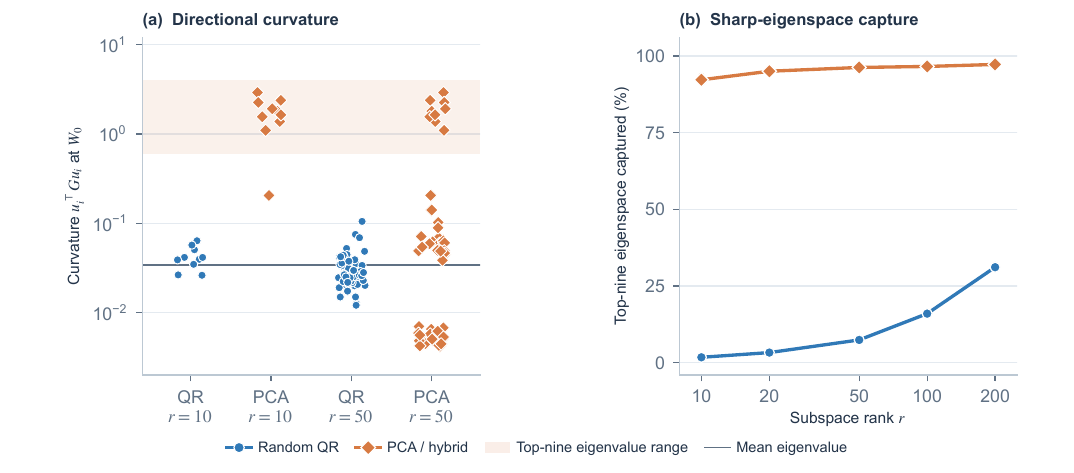}
\caption{Basis geometry on CIFAR-10 with ResNet-20.
(a) Directional Gauss--Newton curvature at $W_0$ on $S$
for seed~42; shading spans the top nine eigenvalues, and
the line marks the mean eigenvalue.
(b) Top-nine eigenspace capture; lines show
three-seed means and faint points show individual seeds.}
\label{fig:qual_basis_geometry}
\end{figure*}

\begin{figure*}[tp]
\centering
\includegraphics[width=\textwidth]{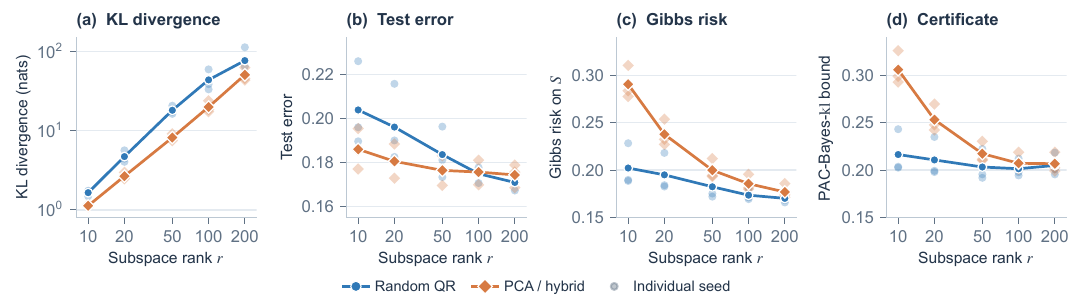}
\caption{Controlled random-QR/PCA comparison on CIFAR-10 with ResNet-20.
(a) KL, (b) mean-predictor test error, (c) empirical Gibbs
risk on $S$, and (d) PAC-Bayes-$\kl$ bound.
Lines show three-seed means; faint markers show individual seeds.}
\label{fig:qual_basis_controlled}
\end{figure*}

\begin{figure*}[t!]
\centering
\includegraphics[width=\textwidth]{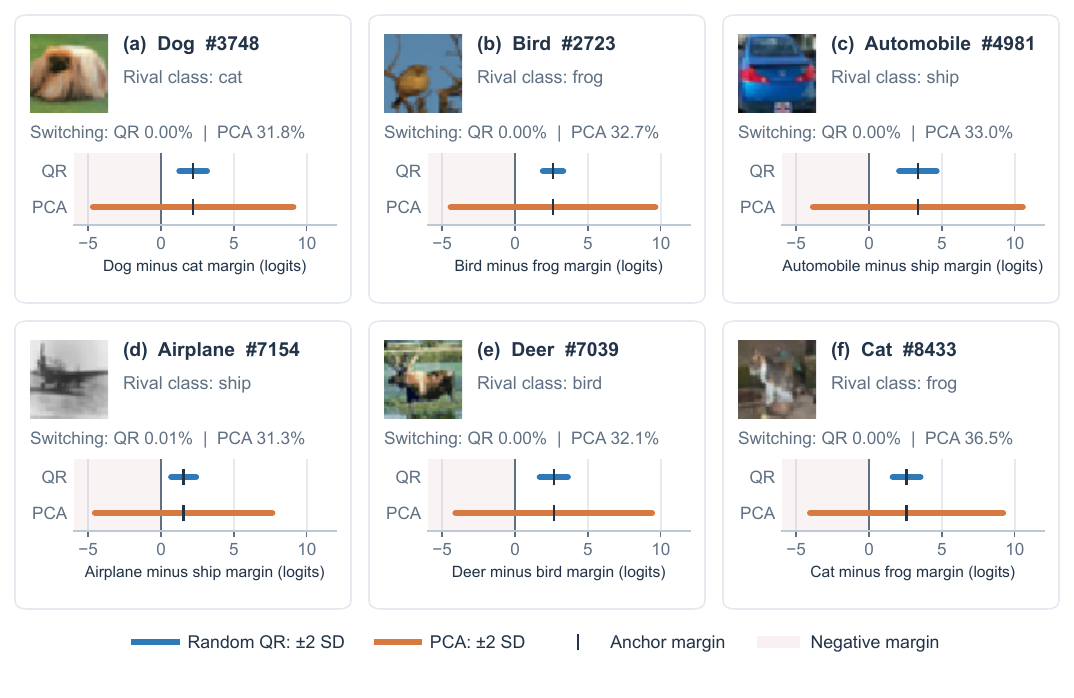}
\caption{Posterior sensitivity of six anchor-correct CIFAR-10
test images with ResNet-20 at Phase~2 initialization
($r=10$, seed~42). Switching rates use 8,192 fresh draws.
Segments show the anchor margin $\pm2$ exact Gaussian SDs
against the most-voted wrong class under PCA in these draws;
shading marks negative margins.}

\label{fig:qual_basis_cases}
\end{figure*}

\Needspace{3\baselineskip}
\paragraph{Qualitative analysis of the random-basis advantage.}
To distinguish basis effects from the differing schedules in
Table~\ref{tab:basis}, we conduct a separate, controlled
CIFAR-10 experiment with ResNet-20~\citep{he2016deep}.
The network uses option-A shortcuts, three residual blocks
at each width of 16, 32, and 64, global average pooling,
and a $10\times64$ head. For seeds $\{42,43,44\}$,
the 50,000 training examples are randomly split into disjoint
halves $S_0$ and $S$. Phase~1 uses the 30-epoch Adam settings
above, with batch size 128 and random-crop and horizontal-flip
augmentation. Freezing follows Algorithm~\ref{alg:subspace};
Phase~2 uses cached, unaugmented features of $S$.
PCA gradients are recorded on all unaugmented $S_0$ after each
epoch by replaying the identical Phase~1 trajectory.
Both bases use 100 Phase~2 epochs at
$r\in\{10,20,50,100,200\}$, with the priors, posterior
parameterization, and $\eta_z$ of Sec.~\ref{sec:subspace}.
They share $W_0$, minibatch order, and posterior noise stream,
so only the basis differs. Certification follows
Algorithm~\ref{alg:subspace}. Test images serve only as
diagnostics, never for training or model or hyperparameter
selection; seed~42 was designated for visualization before training.

\paragraph{Gradient concentration is curvature concentration.}
Let the frozen features be $f(\mathbf{x})\in\R^{64}$ and the softmax
output of the anchor head be $\mathbf{p}(\mathbf{x})$. Following the layer-wise
Gauss--Newton formulation~\citep{botev2017practical}, the
curvature of the mean cross-entropy on $S$ with respect to
the head weights is
\begin{equation}
\label{eq:qual_curvature}
G=\frac{1}{n}\sum_{(\mathbf{x},y)\in S}
\bigl(\operatorname{diag}\mathbf{p}(\mathbf{x})-
\mathbf{p}(\mathbf{x})\mathbf{p}(\mathbf{x})^\top\bigr)
\otimes f(\mathbf{x})f(\mathbf{x})^\top ,
\end{equation}
which here coincides with the Hessian because the logits are
linear in $W$. The PCA input stacks gradients collected while
Phase~1 descended the loss, and the leading PCA directions coincide
with the sharp eigenspace of $G$, consistent with the finding
that gradients concentrate in the span of a few top Hessian
eigenvectors, about as many as there are
classes~\citep{gurari2018gradient}. In the three runs, the top
nine eigenvalues of $G$, corresponding to $K-1=9$, hold
83 to 84\% of the trace of $G$. The rank-10 PCA basis captures
92 to 94\% of the top-nine eigenspace, while a random rank-10
basis captures 1 to 2\%, consistent with the isotropic
expectation $r/(Kd)\approx1.6\%$
(Figure~\ref{fig:qual_basis_geometry}).
At the start of Phase~2, the subspace posterior is isotropic
within the span of $U$, with covariance
$\Sigma=UU^\top/(2r)$ in vectorized weight coordinates. To second order,
the posterior noise raises the expected cross-entropy by
\begin{equation}
\label{eq:qual_noise_cost}
\tfrac12\operatorname{tr}(\Sigma G)=\frac{1}{4r}\sum_{i=1}^{r}u_i^\top G u_i ,
\end{equation}
where $u_i$ is the $i$-th column of $U$. At $r=10$, the predicted increase is
0.40 to 0.51 nats per example for PCA and 0.007 to 0.012 for random
QR, and the MC increase measured on $S$ agrees to
within 7\%. PCA thus concentrates posterior noise in sensitive
directions, while random QR places most noise in flat directions,
consistent with the flat-minima view of non-vacuous PAC-Bayes
bounds~\citep{dziugaite2017computing,neyshabur2017exploring,yang2022data}.
Curvature depends on parameterization~\citep{dinh2017sharp};
our comparison holds the backbone, anchor, and weight
coordinates fixed and changes only the noise subspace.

\paragraph{A controlled basis comparison.}
Under these matched conditions, Figure~\ref{fig:qual_basis_controlled}
separates mean-predictor accuracy from posterior sensitivity.
At $r=10$, PCA has lower KL but higher Gibbs risk: 0.291
compared with 0.202 for random QR. Its mean predictor is more
accurate at every $r\leq50$, yet its certificate is looser,
indicating sensitivity to posterior noise despite good mean predictions.

For Figure~\ref{fig:qual_basis_cases}, we screen anchor-correct
test images at Phase~2 initialization with $r=10$,
$\boldsymbol{\mu}_z=\mathbf{0}$, and $\gamma_i=1$.
Candidates switch in $\leq1\%$ of random-QR draws and
$\geq25\%$ of PCA draws among 2,048 draws per image.
For seed~42, candidates are grouped by unordered pairs of true
and most-voted wrong classes under PCA. We select the image
nearest the median PCA switching rate in each of the six
largest groups and reevaluate with 8,192 fresh draws;
the most-voted wrong class can change between evaluations.

The six selected images switch in at most
0.01\% of random-QR draws but in 31 to 37\% of PCA draws.
PCA spreads their decision margins 5.6 to 10.1 times more widely. Over the full
test set, 1,429, 1,678, and 1,506 anchor-correct images
meet this criterion in the three seeds; no images meet the criterion in
the reverse direction (Figure~\ref{fig:qual_basis_cases}).
With the trained posteriors,
1,261, 1,552, and 1,353 images correctly classified by
both trained mean predictors remain stable under random
QR and switch under PCA; again, no images show the reverse
pattern.

\paragraph{High-rank behavior and protocol dependence.}
Random QR remains tighter at every studied rank under the
controlled protocol (Figure~\ref{fig:qual_basis_controlled}).
The gap shrinks as the 30 PCA directions receive a decreasing
share of total posterior variance, which stays near
$\operatorname{tr}\Sigma=1/2$. By contrast, the high-rank
PCA advantage in Table~\ref{tab:basis} uses a 13-layer CNN,
50 Phase~1 and 150 Phase~2 epochs, and 50 gradient snapshots.
The high-rank ordering therefore depends on the protocol.
At $r=50$, these separate controlled runs give Gibbs risk and
bounds close to Table~\ref{tab:resnet20}, despite larger KL
(18.1 compared with 6.56 nats) from greater mean drift,
$\sum_i\mu_{z,i}^2/(2\tau^2)$. At $n=25000$, this KL
difference accounts for only about 0.005 of the bound.

\section{Discussion}
\label{sec:discussion}

Low-dimensional support, posterior scale, and controlled mean
drift drive the certificate improvement. Adaptive variance offers
a modest gain over fixed variance but no consistent advantage
over random or shuffled controls. Its independent contribution
requires replicated, paired-seed ablations and analysis of $c$.

Basis selection should account for posterior sensitivity:
low-rank PCA concentrates noise in sharp directions, whereas
random QR largely avoids them. The high-rank ordering depends
on architecture and training protocol, motivating matched
comparisons across architectures.

The certificate applies to the Gibbs predictor; guarantees for
the mean predictor require derandomization (Sec.~\ref{sec:pacbayes}).
ImageNet results assess accuracy only, leaving large-scale
MAGDrop-Subspace certification open.

For $L$ stochastic layers with posterior/prior pairs $(Q_l,P_l)$
and product distributions
$Q\coloneqq\bigotimes_{l=1}^L Q_l$ and
$P\coloneqq\bigotimes_{l=1}^L P_l$,
$\KL(Q\|P) = \sum_{l=1}^{L} \KL(Q_l \| P_l)$;
Theorem~\ref{thm:main} and Lemma~\ref{lem:stability} apply.
Valid certification still requires $S_0\perp S$, deterministic
components fixed using only $S_0$, and every $S$-adapted component
represented in $Q$ and its KL. Capturing cross-layer dependencies
and maintaining tight certificates remain practical limitations.

\section{Conclusion}
\label{sec:conclusion}

We presented MAGDrop, an adaptive Bayesian dropout layer
driven by gradient momentum, and MAGDrop-Subspace, a
low-rank posterior for non-vacuous PAC-Bayes certification.
Independent prior construction and shared low-dimensional
support enable tight Gibbs-risk certificates on two CIFAR-10
architectures with a modest mean-predictor accuracy cost.
Controlled basis comparisons connect certificate quality to
posterior-noise orientation, while variance ablations identify
subspace structure as the dominant source of improvement.
Complementary ImageNet-1K accuracy experiments show
compatibility with strong training recipes. These results
support designing posterior geometry explicitly for certification.

\bibliographystyle{IEEEtranN}
\bibliography{references}

\end{document}